\documentclass{article}

\usepackage{microtype}
\usepackage{graphicx}
\usepackage{subfigure}
\usepackage{booktabs} 
\usepackage{natbib}
\usepackage{blkarray}
\usepackage{dsfont}

\usepackage{algorithm}
\usepackage{algorithmic}

\title{A Unified Analysis of Dynamic Interactive Learning}

\author{Xing Gao, Thomas Maranzatto, Lev Reyzin\\
$ $\\
Department of Mathematics, Statistics, and Computer Science\\
University of Illinois at Chicago\\
\texttt{\{xgao53,tmaran2,lreyzin\}@uic.edu}}
\date{}

\usepackage{fullpage}

\usepackage{amsmath}
\usepackage{amssymb}
\usepackage{mathtools}
\usepackage{amsthm}

\usepackage[capitalize,noabbrev]{cleveref}

\theoremstyle{plain}
\newtheorem{theorem}{Theorem}[section]

\newtheorem{lemma}[theorem]{Lemma}
\newtheorem{corollary}[theorem]{Corollary}
\theoremstyle{definition}
\newtheorem{definition}[theorem]{Definition}

\theoremstyle{remark}

\usepackage[]{todonotes}

\begin{document}

\maketitle

\begin{abstract}
In this paper we investigate the problem of learning evolving concepts over a combinatorial structure. Previous work by~\citet{EmamjomehZadehKMS20} introduced dynamics into interactive learning as a way to model non-static user preferences in clustering problems or recommender systems. We provide many useful contributions to this problem.  
First, we give a framework that captures both of the models analyzed by
~\citep{EmamjomehZadehKMS20}, which allows us to study any type of concept evolution and matches the same query complexity bounds and running time guarantees of the previous models.
Using this general model we solve the open problem of closing the gap between the upper and lower bounds on query complexity.  Finally, we study an efficient algorithm where the learner simply follows the feedback at each round, and we provide mistake bounds for low diameter graphs such as cliques, stars, and general $o(\log n)$ diameter graphs by using a Markov Chain model. 
\end{abstract}

\section{Introduction} 
The problem of recommending products or media is ubiquitous in many practical settings such as search engines, online marketplaces, or media streaming services (e.g. Google search, Amazon, Spotify, etc.).  In such settings any algorithm that tries to optimize recommendations receives implicit feedback from users in the system.  This feedback is then used to refine future queries.

Drawing inspiration from 
earlier work on query learning by~\citet{Angluin87}, 
as well as more recent models for interactive clustering~\citep{AwasthiBV17,BalcanB08,LelkesR15},
\citet{EmamjomehZadehK17} considered such product recommendation
problems from the perspective of combinatorial learning, where specific orderings of recommendations are nodes in a (very large) digraph.  In this graph there is a distinguished node that corresponds to the recommendation that the learner wishes to discover.  This can be thought of as the `ideal' ordering of products in a marketplace, or the `best' recommendation in a streaming service.  A directed edge exists between nodes $s$ and $s'$ if the user is allowed to propose $s'$ as a response to $s$.  For example, a user could select two items from a list and propose they need to be swapped in their ideal ordering.  If the learner proposes a node and this is not the target, it receives noisy (random~\citep{EmamjomehZadehK17} or even adversarial~\citep{DereniowskiTUW19}) feedback in the form of an edge on the shortest path from the proposed node and the target.  This form of feedback is 
similar to the correction queries from query learning~\citep{BecerraBonacheDT06}.

\citet{EmamjomehZadehKMS20} subsequently
considered cases when the combinatorial structure itself can evolve over time -- as they noted, some of these
settings resembled earlier work on shifting bandits~\citep{BousquetW02}.
These dynamic settings are also where our results lie, and among
our results, we generalize the work of
\citet{EmamjomehZadehKMS20} and solve some of their 
open problems herein.

\section{Preliminaries} 
\citet{EmamjomehZadehK17} first introduced a static graph model for robust interactive learning, where there is one fixed concept in the concept class, and the learner is trying to learn under noisy feedback.
Later \citet{EmamjomehZadehKMS20} extended the model to dynamic interactive learning, where the target concept can change during learning. Our work is based on the same framework, so we will briefly describe previously defined models and results here.
\subsection{Static model} 
For clarity we first state the static learning model from \citet{EmamjomehZadehK17}, as it is a foundation for later work on dynamic models. 
\begin{definition}[Feedback graph~\citep{EmamjomehZadehK17}]
Define a weighted (directed or undirected) graph $G = (V,E, w)$, where
the vertices represent a set of $n = |V|$ candidate concepts. The edge set $E$ captures all possible corrections a learner can receive: edge $(s,s')$ exists if the user is allowed to propose $s'$ in response to $s$. The edge weights $w$ are given to the learning algorithm, satisfying a key property: if the learner proposes $s$ and the ground truth is $s^* \ne s$, then every correct user feedback $s'$ lies on a shortest path from $s$ to $s^*$ with respect to edge weight $w$.
\end{definition}

Note that we assume that the weighted graph is
given to the algorithm and faithfully represents the underlying problem.


For an undirected graph $G$, let $N_G(v)$ denote the neighborhood of $v$ in $G$. For a directed graph $G$, let $N^{in}_D(v)$ 
be the
in-neighborhood of $v$ in digraph $D$ (including self loops) and  $N^{out}_D(v)$ be the
out-neighborhood of $v$ in digraph $D$ (including self loops).

In the static model there exists a fixed vertex $t \in V$ that the algorithm is attempting to learn over multiple rounds. In each round the learner proposes a vertex $q \in V$ and receives a feedback vertex $z$. If $q = t$, with probability $1 - p$ the learning algorithm receives feedback $q$ indicating the query is correct, and with probability $p$ it receives a feedback $z$ which is adversarially chosen from $N_G(q)$. If $q \neq t$ the algorithm is given a feedback $z \in N_G(q)$ which is incorrect with probability $p$. Crucially both correct and incorrect feedback is adversarial. As discussed in \citet{EmamjomehZadehK17} this implies learning is only feasible when $p < 1/2$.  

Other important definitions used throughout include the collection of concepts that are consistent with a particular feedback, or the version space for a query-feedback pair, as well as the weighted median of the feedback graph, which can be interpreted as the `center of mass' of the graph.

\begin{definition}[Version space~\citep{EmamjomehZadehK17}] If the learner proposes $q$ and receives feedback $z$, let $S_G(q,z)$ be the collection of concepts (nodes) that are consistent with the feedback. Formally, $S_G(q,z) = \{v \mid z \text{ lies on a shortest weighted path from } q \text{ to }v\}$.
\end{definition}

\begin{definition}[Weighted median~\citep{EmamjomehZadehKS16}] Let $L: V\rightarrow \mathbb{R}^{\ge0}$ be a function that assigns likelihood to every vertex in the feedback graph $G=(V,E,w)$. A weighted median $u$ is a vertex that minimizes $\sum_{v \in V} L(v)\cdot w(u,v)$.
\end{definition}
\citet{EmamjomehZadehK17} presented a multiplicative weight update algorithm, which assigns likelihoods for each vertex in the feedback graph, and repeatedly queries the weighted median, which has the property of halving the total likelihood of its version space each round.

\subsection{Dynamic model}
Our paper is concerned with dynamic interactive learning where the target $t$ is allowed to move. Assume that over the $R$ rounds of learning the target moves at most $B$ times and at round $r$ the target is located at some node $t_r$.  Without further assumption on target evolution, \citet{EmamjomehZadehKMS20} showed the following general mistake upper bound.  For the remainder of the paper, $H(p) = p \log
\frac{1}{p} + (1-p) \log \frac{1}{1-p}$ is the entropy.
\begin{theorem}[\citet{EmamjomehZadehKMS20}]
Assume the total number of rounds $R$ is known beforehand. Let $A = V^R$ be the set of all node sequences of length $R$ and let $a^* = \langle t_1,\ldots,t_R \rangle$ be the sequence of true targets throughout the
$R$ rounds.
Let $\lambda : A \rightarrow \mathbb{R}^{\geq 0}$ be a function that assigns non-negative weights to these sequences, such that $\sum_{a \in A}\lambda(a) \leq 1$. 
There is an online learning algorithm that makes at most $$\frac{1}{1-\mathrm{H}(p)} \cdot \log \frac{1}{\lambda\left(a^{*}\right)}$$ mistakes in expectation.
\end{theorem}

While this is a positive result for the mistake bound, it does not guarantee an efficient algorithm as it has to keep track of weights for all sequences $V^R$, and in the worst case the number of sequences is $O(n^R)$ where $n=|V|$. Relatively efficient implementations exist for the following two models without explicitly constructing the $\lambda$ map for each sequence. 

\subsubsection{Shifting target}
In the Shifting Target model there exists an unknown subset of vertices $S \subseteq V$ where $|S| \leq k$, and the learner knows $k$. Target transition is restricted within $S$, viz. $t_r \in S$ for every round $r$. Previous work by \citet{EmamjomehZadehKMS20} proved the following theorem.

\begin{theorem}[\citet{EmamjomehZadehKMS20}]
\label{shiftold}
Under the Shifting Target model, there is a deterministic algorithm that runs in time $O(n^k poly(n))$ and makes at most
\begin{align*}
    \frac{1}{1-\mathrm{H}(p)} \cdot(k \log n+(B+1) \log k+R \cdot \mathrm{H}(B / R))
\end{align*}
mistakes in expectation.
Furthermore, there exists a graph such that every algorithm makes at least
\[
\min \Big\{\frac{1}{1-\mathrm{H}(p)} \cdot[k \log n+(B-2 k+1) \cdot(\log k)]  
    -o(\log n)-B \cdot o(\log k) \, , \, R-o(R) \Big\}
\]
mistakes in expectation.
\end{theorem}

Note in particular the exponential dependence on $k$ in the runtime of the algorithm.

\subsubsection{Drifting target}
In the Drifting Target model, it is assumed that the target slowly evolves over time, and the evolution can be modeled by following the edges of some transition graph defined below.
\begin{definition}[Transition graph~\citep{EmamjomehZadehKMS20}] There exists a known unweighted digraph $G' = (V, E')$, where in particular $G'$ and $G$ have the same vertex set but their edge sets can be different. The target is only allowed to move along edges of $G'$. Formally for every round $r$ the model requires $t_{r+1} \in \{t_r\} \cup N^{out}_{G'}(t_r)$. Let $\Delta$ be the maximum degree in $G'$.
\end{definition}
Previous work by \citet{EmamjomehZadehKMS20} proved the following theorem.

\begin{theorem}[\citet{EmamjomehZadehKMS20}]
\label{driftold}
Under the Drifting Target model, there is a deterministic algorithm that runs in time $poly(n)$ and makes at most $\frac{1}{1-\mathrm{H}(p)} \cdot(\log n+B \cdot \log \Delta+R \cdot \mathrm{H}(B / R))$ mistakes in expectation.. 

Furthermore, there exists a graph such that every algorithm makes at least 
\[    \min \Big\{R-o(R), \frac{1}{1-\mathrm{H}(p)} \cdot\left(\log n + B \log\Delta\right) -o(\log n)-B \cdot o(\log \Delta) \Big\}
\]
mistakes in expectation.
\end{theorem}

Notice that for both Shifting Target and Drifting Target models, there is a gap of $R\cdot H(B/R)$ between the mistake upper bound and lower bound, which remains an open problem in \citet{EmamjomehZadehKMS20}. We will show a new lower bound and close the gap in Section 4.

\section{A unified model}
Our first contribution is to define a more generalized model inspired by the original Drifting Target model, and show that the results from Theorems \ref{shiftold} and \ref{driftold} are both valid under this generalization, thus unifying previous models. The technique used in \citet{EmamjomehZadehKMS20} to efficiently implement the Shifting Target model was to keep track of likelihoods for each subset of nodes instead of each sequence, reducing computational complexity from $O(n^R)$ to $O(n^k)$. Similarly, their efficient implementation for the Drifting Target model keeps track of likelihoods for each node, reducing computational complexity to $O(poly(n))$.   This method requires customization for each transition model, and can become tricky as the models become more complicated. A key motivation for a general model is that it unifies a wide class of transition models, and allows us to easily obtain mistake upper bounds and runtime guarantees based on a single algorithm.

As above let $G = (V, E, w)$ be the graph representing the candidate models (the feedback graph), and $G'=(V', E', \pi)$ be a directed transition graph, representing all possible ways the target might change over time. The key difference is that for each vertex $i \in V$, $V'$ contains possibly duplicated vertices corresponding to the same vertex $i$, denoted by $V'_i := \{u \in V' \mid u $ corresponds to $ i \in V\}$. Define $n' = |V'|$, $\Delta'$ as the max degree of $G'$, and $\pi_{ij}$ as the transition probability in $G'$, where $\pi_{ii} = (1-b)$ and $\pi_{ij} = \pi_{\text{out}} :=\frac{b}{\Delta'}$ for $i \neq j$ under a uniform transition assumption. 

\if{false}
Some key parameters and notations:
\begin{itemize}
  \item Models can change for a maximum of $B$ times over the course of $R$ rounds. Alternatively, at each round the target can change with probability at most $b=\frac{B}{R}.$
  \item $n = |V|,\ n' = |V'|$
  \item $\Delta$ and $\Delta'$ are the max degrees of $G$ and $G'$ respectively. Alternatively let $\pi_{ij}$ denote transition probability in $G'$ where $\pi_{ii} = (1-b)$ and $\pi_{ij} \geq \pi_{\text{out}} :=\frac{b}{\Delta'} $ for $i \neq j$. In the case that neighbors are chosen uniformly, $\pi_{ij} = \frac{b}{\Delta'}$.
  \item $\mathcal{L}_r(i)$ and $L_r(u)$ represent likelihoods of vertices in $G$ and $G'$ in the $r^\text{th}$ round, and $\mathcal{L}_r$ is aggregated from $L_r$ for each $i\in V$: $\mathcal{L}_r(i)=\sum_{u\in V'_i}{L_r(u)}$.
\end{itemize}
\fi

We present a modified version of the algorithm from \citet{EmamjomehZadehKMS20}. In the $r^\text{th}$ round, we keep track of the likelihoods for each vertex $u \in V'$ as $L'_r(u)$, and likelihoods $L_r(i)$ for each vertex $i\in V$ is aggregated from $L'_r$ as the summation over all of $i$'s duplicates in $V'$. The median of the feedback graph $G$ is then calculated based on $L_r$. We update the likelihoods $L'_{r+1}$ for all corresponding nodes in $G'$ based on each node's consistency with the feedback using the same rules as \citet{EmamjomehZadehKMS20}.

\begin{algorithm}[tb]
\caption{Interactive learning likelihood update}\label{generalalgo}
\begin{algorithmic}
\STATE Initialize $L'_1(u)\ \forall u \in V'$
\FOR{$1\leq r \leq R$}
    \STATE $\forall i\in V: L_r(i) \gets \sum_{u\in V'_i}{L'_r(u)}$ \COMMENT{Aggregate $L_r$ from $L'_r$}
    \STATE $q_{r} \gets \arg\min_{i\in V} \sum_{j\in V}{L_{r}(j)}\cdot w(i,j)$ \COMMENT{Query the weighted median}
    \STATE $z_r \gets$ feedback from adversary
    \STATE $\forall i\in V\text{and}\ \forall u\in V'_i:$
    \STATE $\mathrm{\ \ \ \ } P(u)= P(i)\gets (1-p)\cdot \mathds{1} i\in S_G(q_r,z_r)]+p\cdot \mathds{1}[i\not\in S_G(q_r,z_r)]$ \COMMENT{Weight update}
    \STATE $\forall u\in V': L'_{r+1}(u) = \sum_{v \in N^{in}_{G'}(u)}{P(v)\cdot L'_r(v)\cdot \pi_{vu}}$ \COMMENT{Transition}
\ENDFOR
\end{algorithmic}
\end{algorithm}

\begin{theorem}
\label{genthm}
Assuming the first target is chosen uniformly at random from $V'$, Algorithm \ref{generalalgo} runs in time
$O(\Delta' \cdot n'+poly(n))$, uses space $O(n')$, and has query complexity  $$\frac{1}{1-H(1-p)}\cdot \Big(\log{n'}+B\cdot\log{\Delta'}+R\cdot H(B/R)\Big).$$

Alternatively writing the bound using transition probabilities instead of maximum degree, Algorithm \ref{generalalgo} runs in time
$O\left(\frac{1}{n'}\cdot \pi_{\text{out}} ^B\cdot(1-b)^{R-B}\right)$, uses space $O(n')$, and has query complexity  $$\frac{1}{1-H(1-p)}\cdot \Big(\log{n'}+B\cdot\log({b}/{\pi_{\text{out}}})+R\cdot H(b)\Big).$$

\if{false}
Assuming the first target is chosen uniformly random from $V'$, the likelihood of the ground truth sequence is at least: 
$$\lambda(a^*)=\frac{1}{n'\cdot \Delta'^{B}\cdot \binom{R}{B}} = \frac{1}{n'}\cdot \pi_{\text{out}} ^B\cdot(1-b)^{R-B}$$

According to Theorem 5 from \citet{EmamjomehZadehKMS20}, query complexity upper bound is:
\begin{align*}
\frac{1}{1-H(1-p)}\cdot &\Big(\log{n'}+B\cdot\log{\Delta'}+R\cdot H(B/R)\Big) \\ =
\frac{1}{1-H(1-p)}\cdot &\Big(\log{n'}+B\cdot\log{\frac{b}{\pi_{\text{out}}}}+R\cdot H(b)\Big)
\end{align*}

Time complexity is
$O(\Delta' \cdot n'+poly(n)) = O(\frac{1}{\pi_{\text{out}}} \cdot n'+poly(n))$.

Space complexity is $O(n')$.
\fi
\end{theorem}

\begin{proof}
We wish to show that the likelihood of the ground truth sequence $a^*$ is at least $$\lambda(a^*)=\frac{1}{n'\cdot \Delta'^{B}\cdot \binom{R}{B}},$$
or alternatively $$\lambda(a^*) = \frac{1}{n'}\cdot \pi_{\text{out}} ^B\cdot(1-b)^{R-B}.$$
Note that these two expressions correspond to two equivalent interpretations of the transition model: 1, the target changes at most $B$ times during $R$ rounds; 2, the target changes with probability at most $b = B/R$ at each round. 

For the first interpretation, we provide the expression for $\lambda(a^*)$ with a combinatorics argument: there are $n'$ choices for the first node, and the next node differs from the previous node at most $B$ times, each time with $\Delta'$ choices, and these changes can occur at $\binom{R}{B}$ locations in the sequence. Thus the total number of valid sequences is $n'\cdot \Delta'^{B}\cdot \binom{R}{B}$. The initial likelihoods are assigned uniformly among all sequences, so dividing $1$ by the total number of sequences gives us $\lambda(a^*)$.

For the second interpretation, we have a probability argument: the sequence starts with any particular node in the transition graph with probability $\frac{1}{n'}$, and the next node changes to one of the neighbors with probability $\pi_{\text{out}}=\frac{b}{\Delta'}$ for $B$ times, and stays the same with probability $1-b$ for $R-B$ times. Taking the product of these probabilities gives the result.

The bound on query complexity follows by substituting $\lambda(a^*)$ into Theorem 5 of \citet{EmamjomehZadehKMS20}. Note that for large $R$ we approximate $\binom{R}{B}$ by $\left(\frac{R}{B}\right)^B\cdot \big(\frac{R}{R-B}\big)^{(R-B)}$, which contributes to the term $R\cdot H(B/R) = R\log R - B \log B - (R-B)\log(R-B)$ after taking logarithm. 

For algorithmic complexity, steps 3 and 6 both take time $O(n')$, step 4 takes time $O(n^3)$, and step 7 takes time $O(\Delta'\cdot n')$, and $L_r, L'_r$ takes space $n$ and $n'$ respectively. 
\end{proof}

Under our generalized model, any dynamic interactive learning problem can be reduced to defining the feedback graph $G$ to represent the concept class, and defining the transition graph $G'$ to represent the concept evolution. Specifically, the Shifting Target model and Drifting Target model studied in the original paper can be shown as special cases under this general model, and we will show that the general bounds agree with the original results.

\begin{corollary}
In the Drifting Target model, Algorithm \ref{generalalgo} runs in time $O(\Delta \cdot poly(n))$, uses space $O(n)$, and makes at most $$\frac{1}{1-H(1-p)}\cdot\Big(\log{n}+B\cdot\log{\Delta}+R\cdot H(B/R)\Big)$$ mistakes in expectation.
\end{corollary}

\begin{proof}
The transition graph $G'$ is the same as the feedback graph $G$, and transition probability is assumed to be uniform. Thus $n' =n$, and $\Delta'=\Delta$. Plugging into Theorem \ref{genthm} gives the result.
\end{proof}

\begin{corollary}
In the Shifting Target model, Algorithm \ref{generalalgo} runs in time $O(k^2\cdot n^k)$, uses space $O(k\cdot n^k)$, and makes at most $$\frac{1}{1-H(1-p)}\cdot\Big(k\cdot \log{n}+(B+1)\cdot\log{k}+R\cdot H(B/R)\Big)$$ mistakes in expectation.
\end{corollary}

\begin{proof}
The transition graph $G'$ consists of $\binom{n}{k}$ disconnected sub-graphs, where each sub-graph is a clique of size $k$, corresponding to a subset of $k$ vertices in $V$. Each round the target might shift within a $k$-clique, and each clique represents a possible choice of the $k$-subset of targets. Thus $n'= \binom{n}{k} \cdot k$ and $\Delta'=k$. Plugging this into Theorem \ref{genthm} gives the result.
\end{proof}

\if{false}
\subsection{Drifting target model as a special case}
In the drifting target model, the transition graph $G'$ is the same as the feedback graph $G$, and transition probability is assumed to be uniform. Thus $n' =n$, and $\Delta'=\Delta$.
Query complexity: $\frac{1}{1-\mathbf{H}(1-p)}\cdot\Big(\log{n}+B\cdot\log{\Delta}+R\cdot\mathbf{H}(B/R)\Big)$.
Time complexity: $O(\Delta \cdot poly(n))$. Space complexity: $O(n)$. 
\subsection{Shifting target model as a special case}
$G'$ consists of $\binom{n}{k}$ disconnected sub-graphs, where each sub-graph is a clique of size $k$, corresponding to a subset of $k$ vertices in $V$. Each round the target might shift within the $k$-clique, and each clique represents a possible choice of the $k$-subset of targets. $n'= \binom{n}{k}  \cdot k$, $\Delta'=k$.
Query complexity: $\frac{1}{1-\mathbf{H}(1-p)}\cdot\Big(k\cdot \log{n}+(B+1)\cdot\log{k}+R\cdot\mathbf{H}(B/R)\Big)$.
Time complexity: $O(k^2\cdot n^k)$. Space complexity: $O(k\cdot n^k)$. 
\fi

Since the query and computational upper bounds mostly depend on the size of transition graph, namely $n'$ and $\Delta'$, minimality of the transition graph is crucial for query and computational efficiency. We want to find the worst case query upper bound, which can be used as a benchmark when modeling various types of transitions. A trivial upper bound on query complexity occurs in the case that the learner does not have any information about how target might change over time, thus the transition graph $G'$ is a complete graph on $n' = n$ vertices, and $\Delta' = n$. Plugging into Theorem \ref{genthm} gives the following result.
\begin{corollary}
\label{trivial_ub}
The worst case query complexity using Algorithm \ref{generalalgo} is
$$\frac{1}{1-H(1-p)}\cdot\Big((B+1)\cdot\log{n}+R\cdot H(B/R)\Big),$$
and runs in time $O(poly(n))$ and space $O(n).$
\end{corollary}

In the following sections, we will discuss a few examples of other transition models, showing a hierarchy of query complexity.

\subsection{Shortest path }
\if{false}
\begin{definition}
$C_G(s,t)$ 
The set of vertices reachable via a shortest path from $s$ through $t$ is $C_G(s, t)$
\textbf{should this be "vertices on a shortest path from $s$ through $t$" based on description of shortest path transition?}
\end{definition}
\fi
Given two vertices $s,t \in G$, define $S_G^B(s,t)$ to be the collection of all subsets of $S_G(s,t)$ that contain at most $B$ vertices. Formally, $S_G^B(s,t) = \{H \subseteq S_G(s,t) : |H| \leq B \}$. In the Shortest Path model we insist the target can only move along a shortest path in $G$.

We can  describe this  model in the language of our generalized framework. The transition graph $G'$ consists of many disconnected directed paths, each corresponding to some element of $S_G^B(s,t)$ for some $s,t \in G$. This procedure overcounts, so we also restrict $G'$ to only include one copy of any subset of vertices in a path in $G$.  Finally the vertices in any subgraph of $G'$ are connected with $B-1$ arcs that correspond to the ordering imposed by traversing $S_G(s,t)$ from $s$ to $t$. 

The number of vertices in $G'$ is bounded as $n' \leq B \cdot \binom{n}{B} $.  We can't hope to do better than this, as there are classes of graphs with exponentially many shortest paths between two distinguished vertices.  The maximum degree of $G'$ is 2, as all disconnected components are paths.

This model is a variation of the Shifting Target model.  If the target can move $B$ times, then the target can only move in one direction in each valid path.  We can still apply Thm. \ref{genthm} and get a naive mistake upper bound that runs in time $n^B \cdot poly(n)$. We can achieve the $\binom{n}{B}$ bound on the number of subsets when $G$ is a path with $n$ vertices.  However, the target can only move in one direction along the path, so it's natural to think a better algorithm can be developed at least for this case. 

\begin{corollary}
In the Shortest-path model, Algorithm \ref{generalalgo} runs in time $O(n^B)$, uses space $O(B\cdot n^B)$, and makes at most $$\frac{1}{1 - H(p)}\cdot \left(B\cdot \log n + (B + 1) \cdot \log B + R\cdot H(B/R) \right)$$ mistakes in expectation.
\end{corollary}

\if{false}

\subsection{Drift-Shift model}
Next we consider a model that combines drifting and shifting. When the target changes, it can either drift following an original transition graph, or shift to a non-neighboring node among $k$ choices. The entire transition graph $G'$ consists of $\binom{n}{k}$ disconnected sub-graphs, where each sub-graph contains a clique of size $k$, corresponding to a subset of $k$ vertices in $V$. These cliques model the shifting behavior. Each vertex $v_i$ in the clique further connects to a sub-graph $G''_i$, which is the original transition graph for drifting. In addition to the original transition edges (up to degree $\Delta$), each vertex in $G''_i$ has directed edges to all the other $k-1$ vertices $v_{j\neq i}$ in the same clique, representing shifting among $k$ choices. After each shift, drifting will start from the new node. Thus we have $n'= \binom{n}{k} \cdot k\cdot n$ and $\Delta'=\Delta+k$. Directly applying Theorem \ref{genthm} gives the following result.

\begin{corollary}
In the combination Drift-Shift model, Algorithm \ref{generalalgo} runs in time $O(k\cdot(k+\Delta) \cdot n^{k+1})$, uses space $O(k\cdot n^{k+1})$, and makes at most $\frac{1}{1-H(1-p)}\cdot\Big((k+1)\cdot \log{n}+\log{k}+B\cdot\log{(k+\Delta)}+R\cdot H(B/R)\Big)$ mistakes in expectation.
\end{corollary}

\fi

\subsection{m-Neighborhood}
Let $N_G^m(v)$ denote the set of vertices in $G$ that have a shortest path of length $m$ to $v$.  In the m-Neighborhood model, the target can move within $N_G^m$, and $m$ is known to the learner. This model is a variation of the Drifting Target model, and note that $m = 1$ is exactly the case when $G = G'$ in the original Drifting Target model. The transition graph $G'$ is constructed by including an arc from every $v\in V$ to every node in its m-Neighborhood.  Note that $n'=n$ and $\Delta' \leq \Delta^m $. Applying Theorem \ref{genthm} gives the following mistake bound for the m-Neighborhood model:
\begin{corollary}
In the m-Neighborhood model, Algorithm \ref{generalalgo} runs in time $O(\Delta^m\cdot n)$, uses space $O(n)$, and makes at most $$\frac{1}{1 - H(p)}\cdot \left(\log n + B\cdot m \cdot \log (\Delta) + R\cdot H(B/R) \right)$$ mistakes in expectation.
\end{corollary}

To complete our hierarchy, in descending query complexity, we have: Shortest Path model, the original Shifting Target model, the m-Neighborhood model, and the original Drifting Target model.

\section{Query complexity lower bound} 
In this section we close the gap between upper and lower bounds on query complexity, which remained an open problem in \citet{EmamjomehZadehKMS20}. We show a query complexity lower bound that matches the upper bound asymptotically. 

Our result requires some background on the noisy binary search problem.  Here there is a distinguished integer $t$ from the set $\{1,...,m\}$.  In each round $r$, the learner queries some integer $x$. If $x = t$ then the item has been found and the procedure stops.  Otherwise with probability $1 - p$ the learner receives correct feedback of the form $x > t$ or $x < t$.  We make use of the following lower bound.
\begin{theorem}[\citet{EmamjomehZadehKMS20}]
\label{noisybinarysearch}
Every algorithm for the noisy binary search problem requires at least $\frac{\log m}{1 - H(p)}- o(\log m)$
queries in expectation.
\end{theorem}

The idea of our proof is to establish a reduction from noisy binary search: given a noisy binary search problem where the target is uniformly random among $m$ items, we will reduce it to a specific Drifting Target problem under our dynamic interactive learning model. Thus the lower bound on noisy binary search is also a lower bound on interactive learning.

\begin{theorem}
For every $n$ and $\Delta'$, there exists a Drifting Target problem such that every algorithm makes at least 
\[
\frac{1}{1-H(1-p)}\cdot\Big(\log{n}+B\cdot\log{\Delta'}+R\cdot H(B/R)\Big)
- o\left( \log{n}+B\cdot\log{\Delta'}+R\cdot H(B/R)\right)
\]
mistakes in expectation.
\end{theorem}

\begin{proof}
We define the interactive learning problem in the following way: choose $n$ and $R$ such that $m \leq n^R$. Each of the $m$ items can be represented as a base $n$ encoding/enumeration of a sequence with $R$ digits. The feedback graph $G$ for the learning problem is a simple path on $n$ vertices, ordered in the same way as in the encoding: the left end is the smallest digit while the right end is the largest. The transition graph $G'$ is defined on the same set of vertices as in $G$ so $n'=n$. $G'$ includes all the edges in $G$ as bi-directional edges, potentially with additional edges up to some degree $\Delta'$. 

For example, to search among $m\le1000$ items, we can choose $n=10$ and $R=3$, so that each item can be encoded by a length-3 sequence between $\langle0,0,0\rangle$ and $\langle9,9,9\rangle$, which are our familiar base-$10$ natural numbers. The graph $G$ consists of vertices $0,1,2,\ldots,9$ on a path, and interactive learning continues for $3$ rounds. Suppose the target item is encoded as $\langle1,1,2\rangle$, then the ground truth target locations during the 3 rounds are vertices $1,1,2$ respectively.

In the $r^\text{th}$ round of the interaction, the learner queries vertex $i \in [n]$, which can be interpreted as guessing the $r^{\text{th}}$ digit of the target's encoding sequence. Without loss of generality, suppose the adversary's feedback is some vertex $j$ to the left of $i$, which is interpreted as less than $i$. The learner updates likelihoods for sequences whose $r^{\text{th}}$ digit is less than $i$ by a factor of $1-p$, and the other sequences by a factor of $p$. According to Lemma 6 from \citet{EmamjomehZadehKMS20}, the likelihood of the target item's encoding sequence (ground truth) decreases exponentially more slowly than the rest of the sequences and 
will eventually prevail.

To establish the lower bound for a Drifting Target problem with arbitrary $n$ and $\Delta'$, there exists a noisy binary search problem on $m = n\cdot \Delta'^B \cdot \binom{R}{B}$ items that reduces to the Drifting Target problem. The encoding is restricted such that after the first digit is chosen, the remaining digits can change at most $B$ times among $\Delta'$ choices (all other sequences are initialized with 0 likelihoods). Plugging in our value of $m$ into Theorem \ref{noisybinarysearch} gives a mistake lower bound of $\frac{1}{1-H(1-p)}\cdot\Big(\log{n}+B\cdot\log{\Delta'}+R\cdot H(B/R)\Big)- o( \log{n}+B\cdot\log{\Delta'}+R\cdot H(B/R))$, as desired.
\end{proof}

\section{Efficient algorithm for low diameter graphs}
While Algorithm \ref{generalalgo} emphasizes on bounding the number of mistakes for general interactive learning problems, its computation can be inefficient in each round and deteriorates  as the transition model becomes more complex. We realize that the computational complexity mainly comes from keeping track of the likelihoods under all possible transitions, so we consider an alternative approach where the learner ignores the transition model completely and simply follows the adversary's feedback each round.  After the initial query, the algorithm requires no computation. In this section, we study this simple algorithm's performance on low diameter graphs. We formally present this algorithm below.
\begin{algorithm}
\caption{`Follow the Feedback' Procedure for Interactive Learning}
\label{simplealgo}

\begin{algorithmic}
\STATE $q_1 \gets \mathrm{argmin}_{i\in V} \sum_{j\in V} w(i,j)$\COMMENT{Start with a `center' vertex}
\FOR{$1\leq r \leq R$} 
\STATE $z_r \gets$ feedback from adversary after querying $q_r$
\STATE $q_{r+1} \gets z_r$ \COMMENT{Follow the feedback for next round}
\ENDFOR{}
\end{algorithmic}
\end{algorithm}

\subsection{Cliques: graphs with diameter 1}
A clique is the most symmetric graph, where each vertex can be considered the center, and the graph has diameter 1. This means no matter which node the learner queries, a correct feedback from the adversary will reveal the true target at each round. Therefore after each mistake, the learner will keep querying the correct node until the target's next move. The mistake upper bound is stated in the theorem below.
\begin{theorem}
\label{d1thm}
If the feedback graph $G'$ is fully-connected (a clique on the concept class), Algorithm \ref{simplealgo} makes at most $B+p(R-B)$ mistakes in expectation.
\end{theorem}
\begin{proof}
By assumption the learner queries a node then receives a feedback, and the target may move at any point during this process. To help with the analysis in this case, we break down the chain of events in the following way: in each round we assume that at first the target moves (or stays put), then the learner makes a query and receives a new feedback. Notice that in every round where the target moves the learner will make a mistake regardless of the correctness of the previous feedback. If we assume target can move at most $B$ times, this leads to $B$ mistakes. For the $R-B$ rounds where the target doesn't move, the learner makes a mistake if and only if the previous feedback is incorrect. As feedback is noisy with probability $p$, this leads to $p(R-B)$ mistakes. So the expected number of mistakes over the course of $R$ rounds is: 
\begin{equation*}
E[M] = B + p(R-B)
\end{equation*}
In anticipation of discussing other classes of graphs we present a second analysis of Algorithm \ref{simplealgo} on cliques. Assume that each round the target moves with probability $b=\frac{B}{R}$. We can model the process as a Markov Chain where the states $\{0, 1\}$ represent the learner's distance from the target at each round. Note that these states do not represent the learner's position in the graph. Now we break down the chain of events in a slightly different way: first, the learner queries the node received from previous feedback and receives a new feedback, then target either moves or stays put for the next round.

The new feedback is correct with probability $1-p$, and the target stays at the same vertex with probability $1-b$, so in the next round, the learner queries the correct vertex (transitions to state 0) with probability $(1-p)(1-b)$. If either the new feedback is incorrect or the target moves at the end of this round, the learner will make a mistake next round (transitions to state 1) with probability $p+b-pb$, assuming noise of feedback and target evolution are independent. The state transition matrix is:
\[ P=
\begin{blockarray}{ccc}
& 0 & 1 \\
\begin{block}{c(cc)}
  0 & (1-p)(1-b) & p+b-pb\\
  1 & (1-p)(1-b) & p+b-pb\\
\end{block}
\end{blockarray}
 \]
Each row in the transition matrix is already in its stationary distribution $\pi = (\pi_0,\pi_1)$. The expected number of mistakes over the course of $R$ rounds is: $E[M] = R(1-\pi_0) = B + p(R-B)$.
\end{proof}

\subsection{Stars: graphs with diameter 2}
A simple star graph is a graph of diameter 2, with one center vertex connecting to all the other vertices (leaf nodes). After querying the center vertex, a correct feedback will reveal the true target, so an efficient strategy is to query the center first then follow the feedback. We assume the target only moves among the leaf nodes, because the learner will make no more mistakes in the case that the target can move to the center: if the learner queries a wrong leaf, it takes at least 2 queries if target is on another leaf, and takes 1 query if the target is at the center.

\begin{theorem}
\label{d2thm}
If the feedback graph $G'$ is a star then Algorithm \ref{simplealgo} makes at most $2B + p(R-B)+p^2(R-B)$ mistakes in expectation.
\end{theorem}
\begin{proof}
We again break down the chain of events in this order: first, the target either moves or stays put, then the learner queries the previous feedback received and the adversary provides a new feedback. For the $B$ rounds that the target shifts, the learner will make $1$ mistake each time. For the $R-B$ rounds when the target doesn't move, if the previous feedback was incorrect, the learner will make $1$ mistake; if the previous feedback was correct, the learner will make a mistake if the feedback pointed to the center, which means the previous query was a wrong leaf. Another case that the feedback points to the center is when the learner queried the correct leaf, but received an incorrect feedback.

Let $x, y, z$ represent the number of times the learner queries the correct leaf, the wrong leaf, and the center respectively. Based on the analysis above, we can set up a system of linear equations: 
\begin{align} 
x + y + z &=  R\label{eq_1} \\ 
px + (1-p)y &= z \label{eq_2}\\
B+ p(R-B) + (1-p)(R-B)(y/R) & = R-x \label{eq_3}
\end{align}

Equation \ref{eq_1} is trivial; equation \ref{eq_2} represents the number of times the learner queries the center as a function of queries to correct/incorrect leaf nodes; equation \ref{eq_3} is the expected number of mistakes, which is the number of times the learner does not query the correct leaf. After elimination, equation \ref{eq_3} becomes:
\begin{equation*}
\label{eq_star}
\begin{split}
E[M] &= B + p(R-B) + \bigg(  [B+p^2(R-B)] 
 \cdot \frac{(1-p)(R-B)}{B+p^2(R-B) +(1-p)R} \bigg) \\
&\le 2B + p(R-B)+p^2(R-B)
\end{split}
\end{equation*}
Alternatively, if we assume each round the target moves with probability $b=\frac{B}{R}$, we can model the process using a Markov Chain with states $\{0, 1, 2\}$ representing the learner's distance from the target at each round. Similar to the analysis of the clique, we have state transition matrix: 
\[ P=
\begin{blockarray}{cccc}
& 0 & 1 & 2 \\
\begin{block}{c(ccc)}
  0 & (1-p)(1-b) & p & (1-p)b\\
  1 & (1-p)(1-b) & 0 & p+b-pb\\
  2 & 0 & 1-p & p\\
\end{block}
\end{blockarray}
\]
This is a fully-connected Markov Chain, and the stationary distribution $\pi = (\pi_0,\pi_1, \pi_2)$ can be calculated numerically. The expected number of mistakes over the course of $R$ rounds is $R(1-\pi_0)$. It can be verified that the numerical solution agrees with the analytical solution above.
\end{proof}

In Appendix~\ref{dstar}, we extend this analysis for ``quasi-stars'' with
a central vertex connecting otherwise disjoint paths of 
length $d/2$ (for even $d$).

\subsection{Graphs with diameter o(log n)}

\if{false}
Assume $G$ has diameter $d$ Similar to the star graph, consider the algorithm where the learner simply follows the feedback given.  Due to the diameter assumption, no feedback will move us more than $d$ nodes away from the target.  

We can model this process as a Markov chain on a path with $d$ nodes, forward transition probabilities $1-p$ and reverse transition probabilities $p$.  Each node in the path represents the learner's distance from the target node.  It is a well known fact that in a path with $d$ nodes and all transition probabilities $0.5$ the hitting time from the first node to the last node is $O(d^2)$ (how to incorporate $p \not= 0.5$ in the bound?).  Thus if the target doesn't move the learner will make $O(d^2)$ mistakes before converging.  If the target can move $k$ times, then the learner makes $O(d^2 k)$ mistakes.

Noting that the trivial mistake bound for efficient shifting given in Section \ref{Trivial Bound}, setting $d = o\!\left(\sqrt{\log n}\right)$ when $B \approx k$ gives better than trivial mistake bounds.  If $B \ll k$ then once we find all the vertices in $B$ we can noisy binary search between them.
\fi

From our previous analysis, we notice that the mistake bound does not depend on the number of nodes in the feedback graph, but rather the diameter, which is the largest distance from any node to the target. Therefore we consider general graphs bounded by diameter $d$. The mistake bound is stated as the theorem below.
\begin{theorem}\label{dthm}
If the feedback graph has diameter $d$, then Algorithm \ref{simplealgo} makes at most $$\frac{1}{1-p}\cdot \Big(dB - \frac{pB}{1-2p} +pR\Big)$$ mistakes in expectation.
\end{theorem}

We model the learning process as a random walk on a Markov Chain with states $\{0,...,d\}$. However, now we reverse the meaning of the states: state $0$ means the query node is distance $d$ from the true target, and state $d$ means the query node is the target. This change does not affect the result of analysis, but greatly simplifies the notation. Every time the target moves, the random walk restarts at state $0$ and moves towards state $d$: the learner moves 1 step forward upon every correct feedback, and moves 1 step backward upon every noisy feedback. There are two types of mistakes during the random walk: before reaching the target for the first time, every query contributes a mistake; once the learner reaches the target, it will circle around it due to noise probability $p<1/2$, and occasionally misses the target.

The first type of mistake is captured by the hitting time of random walk on the Markov Chain from state $0$ to state $d$. We have the following lemma (see Appendix~\ref{s:h0d_lemma} for the proof):
\begin{lemma}
\label{h0d_lemma}
Let $p < 1/2$, for a Markov Chain on a path of length $d+1$, the random walk with forward probability $1-p$ and backward probability $p$ has a hitting time $$h_{0,d} \le \frac{d}{1-2p} - \frac{p}{(1-2p)^2}.$$
\end{lemma}

Next we consider the second type of mistake. Once the learner reaches the target, it will 
keep reporting the correct node unless it receives noisy feedback and is misguided to move away from the target, which will cause a mistake for the next query. We bound the fraction of time the learner misses the target with the following lemma. 
\begin{lemma}
\label{toff_lemma}
Let $p<1/2$, after reaching state $d$ and before the next target transition, the expected fraction of time the learner wanders away from state $d$ is bounded by $$t_{\text{off}} \le \frac{p}{1-p}.$$
\end{lemma}
\begin{proof}
Once the random walk reaches state $d$ (learner queried the correct target), let $t_d$ denote the expected time spent at state $d$, we have the following recurrence relations:
\begin{align*}
p\cdot t_d = (1-p)\cdot t_{d-1} &\implies t_{d-1} = r\cdot t_d\\
t_{d-1} = (1-p)\cdot t_{d-2} + p\cdot t_d &\implies t_{d-2} = r\cdot t_{d-1}\\
...\\
p\cdot t_1 = (1-p)\cdot t_0 &\implies t_0 = r\cdot t_1 \\
\text{For } i=0...d: t_i &= r^{d-i} \cdot t_d
\end{align*}

The expected fraction of time not spent at state $d$:
\begin{align*}
    t_{\text{off}} &= 1-\frac{t_d}{\sum_{i=0}^{d}{t_i}} \\
    &= 1 - \frac{1-r}{1-r^{d+1}} \le r \\
    &=  \frac{p}{1-p},
\end{align*}
which finishes the proof.
\end{proof}

Note that the hitting time $h_{0,d}$ is linear in $d$, and $t_{\text{off}}$ is positively related to entropy $H(p)$. Combining the results above, we can prove our theorem:
\begin{proof}[Proof of Theorem \ref{dthm}]
Assume every time the random walk restarts at state $0$, state $d$ can be reached before the next restart. This means every time the target moves, the learner is able to reach the target before its next transition. Since the learner makes a mistake every round spent on the hitting time, this is the worst case assumption because the learner is forced to make all the mistakes possible for each target transition. Combining the two types of mistakes from previous lemmas, the total expected number of mistakes is:
\begin{equation*}
\label{eq_diameter-d}
\begin{split}
E[M] &= B\cdot h_{0,d} + (R- B\cdot h_{0,d})\cdot t_{\text{off}}\\
&\le B\cdot \Big( \frac{d}{1-2p} - \frac{p}{(1-2p)^2} \Big)\cdot\frac{1-2p}{1-p} + \frac{pR}{1-p}\\
&= \frac{1}{1-p}\cdot \Big(dB - \frac{pB}{1-2p} +pR\Big),
\end{split}
\end{equation*}
which completes the proof.
\end{proof}

In the case that $d=2$, the bound in Theorem \ref{dthm} for a general diameter-$2$ graph is slightly larger than the bound from Theorem \ref{d2thm} for the star graph. This makes sense because a star is the best case diameter-$2$ graph, with a center node that when queried provides information to the true target. 

In the case that $d=o(\log{n})$ and $p=o(H(B/R))$, we notice that the result from Theorem \ref{dthm} is comparable to the trivial upper bound of Algorithm \ref{generalalgo} as stated in Corollary \ref{trivial_ub}. This means that if the learner has very limited information on target transition, or the transition model is complex, and the graph is bounded by low diameter, then Algorithm \ref{simplealgo} makes a huge improvement on computational efficiency without too much sacrifice on query complexity. Note that a complex transition model is often correlated with a low diameter feedback graph: highly connected graphs tend to have low diameters, and potentially complex transitions due to the close relationships between concepts.

\subsection{Paths: graphs with diameter $n$}
\label{npath}

We also note that while path graphs seem like an easy case, they actually present difficulties due to
their large diameter.  

An upper bound on noisy binary search was given by~\citep{BenOrH08}. Their algorithm returns the  correct element with probability $(1 - \delta)$ with an expected
$$\frac{(1 - \delta)}{1 - H(p)}\cdot\big(\log n + O(\log \log n) + O(\log(1/\delta))\big)$$
queries.  This can be implemented in poly-time.  

A naive algorithm for the shifting target case is to run their algorithm $k$ times, setting $\delta$ appropriately small, for example, $\delta = 1/ \log(kn)$.  Then as both $k$ and $n$ go to infinity, the probability of failure goes to $0$.  

If $k \approx \log(n)$, $B \approx k$, and the number of rounds is much larger than the expected number of queries, this naive algorithm essentially matches the mistake bound from~\cite{EmamjomehZadehK17}.  The difference is the $k \log k$ vs. $k^2 \log k$ and $R \cdot H(B/R)$ vs. $\log(\log(kd))$ terms.

\section{Acknowledgements}

This work was supported in part by the National Science Foundation grant CCF-$1934915$.


\if{false}

\subsection{Different transition models in d-dimension grid}
Change of notation: total number of nodes in $G$ is $N:=|V|$ , and $N=n^d, \Delta = 2d$. Similarly, total number of nodes in $G'$ is $N'=|V'|$. Whether $G'$ is also a grid and its structure depends on the transition model. 

\subsubsection{Drifting target model}
In the simple drifting model, $G'=G, N'=N=n^d, \Delta'=2d$.

Query complexity: $\frac{1}{1-\mathbf{H}(1-p)}\cdot\Big(\log{N}+B\cdot(\log{d}+1)+R\cdot\mathbf{H}(B/R)\Big)$.

Time complexity: $O(2dN+poly(N))$. Space complexity: $O(N)$.

\subsubsection{Shifting target model}
$G'$ is not a simple grid, but can be thought of as a $\binom{N}{k}$ rows by $k$ columns grid, where each row is fully connected horizontally, but disconnected vertically. $N'= \binom{N}{k} \cdot k$, $\Delta'=k$.

Query complexity: $\frac{1}{1-\mathbf{H}(1-p)}\cdot\Big(k\cdot \log{N}+(B+1)\cdot\log{k}+R\cdot\mathbf{H}(B/R)\Big)$.

Time complexity: $O(k^2\cdot N^k + poly(N))$. Space complexity: $O(k\cdot N^k)$. 

\subsubsection{Drifting target model on "fully-connected" grid}
In the simple drifting target model, each direction of the grid is a simple path, so target can drift one step at each transition. We can generalize the number of steps to $s<n$, in other words, each node in the transition graph is fully-connected to a $s$-neighborhood on the grid and the target can drift a maximum of $s$ steps in a single direction each round. $N'=N,\Delta'=2sd$, and in the extreme case that the number of steps is not limited, $\Delta'\approx nd$.

Query complexity: $\frac{1}{1-\mathbf{H}(1-p)}\cdot\Big(\log{N}+B\cdot(\log{s}+ \log{d}+1)+R\cdot\mathbf{H}(B/R)\Big)$.
Extreme case: $\frac{1}{1-\mathbf{H}(1-p)}\cdot\Big((1+\frac{B}{d})\log{N}+B\cdot\log{d}+R\cdot\mathbf{H}(B/R)\Big)$.

Time complexity: $O(2sdN+poly(N))$.
Extreme case: $O(ndN+poly(N))$.
Space complexity: $O(N)$.

\subsection{Efficient versus inefficient algorithms}
There are two parts that contribute to the algorithm's time complexity: O($\Delta'\cdot N')$, which comes from likelihood update for each node in $G'$, and $O(poly(N)) (?N^3)$, which comes from computing the weighted median in $G$.

Previously attempted to improve $O(poly(N)) = O(poly(n^d))$ to $O(d\cdot n)$ for the drifting target model, but failed. The idea was to find median coordinate in each direction independently and query the node indexed by the medians. The feedback differs from the query node in one direction, which can be interpreted as a hyper plane that cuts the grid in half from the query node. But weight update cannot be based on which side of the hyper plane a node is at, because correct weight update for each node depends on all directions.

Another attempt is to improve O($\Delta'\cdot N') = O(k^2\cdot N^k)$ to $O(k\cdot N)$ in the shifting target model. The idea is that each node appears in exactly the same number of cliques in $G'$, and in the likelihood update summation for a node $v$, all nodes $u\neq v \in V' $ appears in exactly the same number of times, in other words, a constant fraction of $u\neq v \in V' $ contribute to the likelihood of $v$. But this idea fails for a similar reason as in the previous case: even though the fraction of nodes is constant for all nodes, their likelihoods differ based on feedback history.

Question: can we improve at all? Assuming the transition graph is minimal, which means each duplicated node represents a unique transition possibility, the coarsest level of computation is at the node level in both $G$ and $G'$?

\subsection{Random vs. adversarial noise}
\section{lower  bound w/ random feedback}

what happens if adversary gives random neighbor on grid in case of noise?  what is his normal difficult strategy?

What about directed expanders? Assume max (weighted) cycle length $c$, 
$\frac{1}{1-\mathbf{H}(p)}$ becomes $\frac{1}{-\log{\tau}-\mathbf{H}(p)}$, where $\tau = \frac{c-1}{c}\cdot(1-p)+\frac{1}{c}\cdot p$. (In undirected graph, $\tau = \frac{1}{2}$).

\fi

\if{false}
\section{Lit Review}
\subsection{A General Framework for Robust Interactive Learning}
\begin{itemize}
\item provides a framework for online learning outside of binary classification
\item In their framework there is a single target $s^*$ to be learned that exists in a graph with many other candidate structures
\item The weighted graph $G$ contains an edge between structures $s, s'$ if $s'$ is valid feedback for when $s \not= s^*$
\item At each time step the learner proposes a structure, and receives feedback as a shortest edge to the target structure
\item noisy feedback: with probability $p$ the learner receives correct (and adversarial) feedback.  With probability $1-p$ the learner receives arbitrary adversarial feedback
\item The algorithm is a modified multiplicative weights update that uses the directional feedback to set the likelihoods
\item the algorithm given is sample efficient, and can be made computationally efficient with monte-carlo methods
\item Gives three (efficient!) applications to learning rankings, clusterings, and binary classifiers
\end{itemize}

\subsection{Interactive Learning of a Dynamic Structure}
\begin{itemize}
\item Proposed an extention of above paper: what if the target is able to move in $G$?
\item Same setting as above, except that the target is allowed to move up to $B$ times before finally stopping at some target
\item the algorithm given to solve this is another modified MW update, and is sample-efficient but not computationally efficient if we place no restrictions on how the target can move
\item Proposed the shifting target model and drifting target model as examples of target movement models that can be made efficient
\begin{itemize}
    \item STM has the target move in a pre-determined subset of $k$ nodes (subset not known to learner, $k$ is).  Efficient implementation except for an exponential dependence on $k$
    \item DTM has the target move in a pre-determined graph $G'$ that can be different to $G$ ($G'$ is known to the learner). Fully efficient algorithm for this case
\end{itemize}
\end{itemize}
\fi

\newpage

\bibliography{main}
\bibliographystyle{abbrvnat}

\newpage
\appendix
\onecolumn

\section*{APPENDIX}

\section{Proofs of technical lemmas}\label{s:techproofs}

\subsection{Proof of Lemma~\ref{h0d_lemma}}

\begin{proof}\label{s:h0d_lemma}
The transition probabilities of the Markov Chain are:
\[
P_{ij}=
\begin{cases}
1-p & j=i+1 \text{ (move towards target)}\\
p & j=i-1 \text{ (move away from target)} \\
1-p & j=i=d \text{ (self-loop on the target)}\\
p & j=i=0 \text{ (self-loop on nodes furthest from the target)}
\end{cases}
\]
Let $r = \frac{p}{1-p}$. It follows from the assumption $p<\frac{1}{2}$ that $r<1$. The hitting time analysis follows:
For $i=1 \dots d$: 
$h_{i,i+1} = (1-p)\cdot 1 + p\cdot(1+h_{i-1,i+1}) = 1 + p\cdot(h_{i-1,i}+h_{i,i+1})$\\
We solve the recurrence: $h_{i,i+1} = \frac{1+p\cdot h_{i-1,i}}{1-p}$ with
base cases: $h_{0,1} = \frac{1}{1-p}, \;\;\;\;h_{1,2} =\frac{1+\frac{p}{1-p}}{1-p} = \frac{1}{(1-p)^2}$.\\
For $i \ge 2$: 
\begin{align*} 
h_{i,i+1} &= \frac{\big(\sum_{j=0}^{i-2}{(1-p)^{i-j}\cdot p^j}\big) + p^{i-1}}{(1-p)^{i+1}}\\
&= \frac{(1-p)^{i}\cdot \sum_{j=0}^{i-2}{\big(\frac{p}{1-p}\big)^j} + p^{i-1}}{(1-p)^{i+1}}\\ 
&= \frac{1}{(1-p)^{i+1}}\cdot \Big( (1-p)^{i}\cdot \frac{1-r^{i-1}}{1-r} + p^{i-1}\Big)\\ 
&= \frac{1}{(1-p)^{i+1}}\cdot \Big( (1-p)^{i+1}\cdot \frac{1-r^{i-1}}{1-2p} + p^{i-1}\Big)\\ 
&= \frac{1}{1-2p} -\frac{r^{i-1}}{1-2p} +\frac{p^{i-1}}{(1-p)^{i+1}}\\
&= \frac{1}{1-2p}+ \Big(\frac{1}{(1-p)^2}-\frac{1}{1-2p} \Big)\cdot  r^{i-1}\\
h_{2,d} &= \sum_{i=2}^{d-1}{h_{i,i+1}}\\
&= \frac{d-2}{1-2p} + \Big(\frac{1}{(1-p)^2}-\frac{1}{1-2p} \Big)\cdot \sum_{i=2}^{d-1}{r^{i-1}}\\
&= \frac{d-2}{1-2p} + \Big(\frac{1}{(1-p)^2}-\frac{1}{1-2p} \Big)\cdot \frac{r-r^{d-1}}{1-r}\\
&= \frac{d-2}{1-2p} + \Big(\frac{1}{(1-p)^2}-\frac{1}{1-2p} \Big)\cdot \frac{p}{1-2p}\\
h_{0,d} &=h_{0,1} + h_{1,2} + h_{2,d}\\
& \leq \frac{1}{1-p} + \frac{1}{(1-p)^2} + \frac{d-2}{1-2p} + \Big(\frac{1}{(1-p)^2}-\frac{1}{1-2p} \Big)\cdot \frac{p}{1-2p}\\
&= \frac{d-2}{1-2p} - \frac{p}{(1-2p)^2} + \frac{1}{1-p}+\frac{1}{(1-p)^2}+\frac{p}{(1-p)^2(1-2p)}\\
&= \frac{d}{1-2p} - \frac{p}{(1-2p)^2}
\end{align*}
\end{proof}

\section{Quasi-stars: graphs with diameter $d$}
\label{dstar}
Now we consider a star graph where each branch is a path of length greater than one, and we have diameter $d>2$. We can generalize the Markov Chain with states $\{0, 1,...,d\}$, representing the distance from query node to the true target. Further assume that every time the target moves, it moves for a distance of at least 2, with uniform probability of landing at any distance ($\ge2$) to the target. 
The $(d+1)$ by $(d+1)$ transition matrix $P$ can be approximated as follows:
\[
P_{ij}=
\begin{cases}
0 & j = i-2, \text{ moves 2 steps closer}\\
(1-p)(1-b) & j = i-1, \text{ moves 1 step closer}\\
0 & j = i,  \text{ distance to target does not change}\\
p(1-b) & j = i+1, \text{ moves 1 step further}\\
p' & \text{all remaining probabilities sum to 1 uniformly}\\
\end{cases}
\] 
With the exception that $P_{00} = (1-p)(1-b)$. For example, for $d=4$:
$$P = 
\begin{pmatrix}
  (1-p)(1-b) & p(1-b) & b/3 & b/3 & b/3\\
  (1-p)(1-b) & 0 & p(1-b) & b/2 & b/2\\
  0 & (1-p)(1-b) & 0 & p(1-b) & b\\
  b & 0 & (1-p)(1-b) & 0 & p(1-b)\\
  \frac{(p+b-pb)}{2} & \frac{(p+b-pb)}{2} & 0 & (1-p)(1-b) & 0\\
\end{pmatrix}$$
With stationary distribution $\pi=(\pi_0,...,\pi_d)$, we get expected total mistakes as $E[M] = R(1-\pi_0)$, which can be computed numerically.

\end{document}